%% file: main.tex
\documentclass{article}
\usepackage{spconf,amsmath,graphicx}
\usepackage{hyperref}       
\usepackage{url}            
\usepackage{booktabs}       
\usepackage{amsfonts}       
\usepackage{nicefrac}       
\usepackage{microtype}      
\usepackage{xcolor}         

\usepackage{algorithm}
\usepackage{algorithmic}
\usepackage{cite}

\usepackage{subcaption}

\usepackage{balance}

\input{math_commands}

\ninept

\renewcommand\paragraph[1]{\subsection{#1}}

\title{Embedding Signals on Graphs with Unbalanced\\Diffusion Earth Mover's Distance}

\name{\parbox{\textwidth}{\centering%
    Alexander Tong$^{1,*}$%
    \quad Guillaume Huguet$^{2,*}$%
    \quad Dennis Shung$^{3,*}$%
    \quad Amine Natik$^{2}$%
    \quad Manik Kuchroo$^{4}$
    \quad Guillaume Lajoie$^{2}$%
    \quad Guy Wolf$^{2,\dagger}$%
    \quad Smita Krishnaswamy$^{4,1,\dagger}$}\thanks{$^{*}$ Equal contribution; $^{\dagger}$ Equal senior author contribution. This research was partially funded by NIH grant K23DK125718-01A1 [\emph{D.S.}]; IVADO PhD Excellence Scholarship [\emph{A.N.}]; CIFAR AI Chair, NSERC Discovery grant 03267 [\emph{G.W.}]; Chan-Zuckerberg Initiative grants 182702 \& CZF2019-002440 [\emph{S.K.}]; NSF career grant 2047856 [\emph{S.K.}]; Sloan Fellowship FG-2021-15883 [\emph{S.K.}]; and NIH grants R01GM135929 \& R01GM130847 [\emph{G.W., S.K.}]. The content provided here is solely the responsibility of the authors and does not necessarily represent the official views of the funding agencies. Correspondence to \textless{}guy.wolf@umontreal.ca\textgreater{} and \textless{}smita.krishnaswamy@yale.edu\textgreater{}}%
}
\address{%
    $^{1}$ Yale University, Dept. of Comp. Sci.; $^{3}$ Dept. of Medicine; $^{4}$ Dept. of Genetics, New Haven, CT, USA \\%
    $^{2}$ Universit\'{e} de Montr\'{e}al, Dept. of Math. \& Stat.; Mila -- Quebec AI Institute, Montreal, QC, Canada\\%
}

\begin{document}

\maketitle

\begin{abstract}
In modern relational machine learning it is common to encounter large graphs that arise via interactions or similarities between observations in many domains. Further, in many cases the target entities for analysis are actually signals on such graphs. We propose to compare and organize such datasets of graph signals by using an earth mover's distance (EMD) with a geodesic cost over the underlying graph. Typically, EMD is computed by optimizing over the cost of transporting one probability distribution to another over an underlying metric space. However, this is inefficient when computing the EMD between many signals. Here, we propose an unbalanced graph EMD that efficiently embeds the unbalanced EMD on an underlying graph into an $L^1$ space, whose metric we call {\em unbalanced diffusion earth mover's distance (UDEMD)}. Next, we show how this gives distances between graph signals that are robust to noise. Finally, we apply this to organizing patients based on clinical notes, embedding cells modeled as signals on a gene graph, and organizing genes modeled as signals over a large cell graph. In each case, we show that UDEMD-based embeddings find accurate distances that are highly efficient compared to other methods.
\end{abstract}


\section{Introduction}

The task of comparing probability distributions is applicable to a wide variety of machine learning problems, giving rise to popular $\phi$-divergences such as the Kullback-Leibler (KL), Hellinger, or total variation (TV) divergences, which ignore the underlying geometry of their support. The Earth Mover's Distance (EMD), also known as the Monge-Kantorovich or Wasserstein Distance, explicitly takes into account this underlying geometry via a domain-specific ground distance, which has many advantages on empirical probability distributions~\cite{peyre_computational_2019, backurs_scalable_2020}. Here, we show that earth mover’s distances are useful in a new domain: that of graph signals. In modern relational machine learning, we encounter large graphs that arise via interactions between entities in many domains~\cite{bojchevski_deep_2021, snomed}. Features of such entities can be considered as signals on the graph. For such signals, which often tend to be noisy, we propose a new unbalanced graph earth mover’s distance, and use it to organize the signals and determine relationships between them.

Since graphs can contain tens (Cora)~\cite{bojchevski_deep_2021} to hundreds of thousands of nodes  (SNOMED-CT)~\cite{snomed}, there is a great need for this measure to be computationally efficient. While the Wasserstein distance is intuitively attractive, it presents computational challenges. Here, based on the recent diffusion EMD method~\cite{tong_diffusion_2021}, we show that an efficient unbalanced EMD between signals can be computed as the difference between graph convolutions of the signal with multiscale graph kernels. This unbalanced EMD can be computed in linear time with convergence guarantees and without solving an optimization problem. We call our distance {\em unbalanced diffusion earth mover's distance} (UDEMD).

While previous work on Wasserstein distance embedding mostly focused on its relation to the balanced optimal transport problem~\cite{indyk_fast_2003, shirdhonkar_approximate_2008, gavish_multiscale_2010, le_tree-sliced_2019, tong_diffusion_2021}, we propose an unbalanced Wasserstein embedding approach between large number of distributions defined as signals on graphs. Since graph signals tend to be noisy, an \textit{unbalanced} transport, which can choose not to transport parts of the data space when it is inefficient to do so, leads to more robust distances between graph distributions that are less sensitive to outliers in the signal.

We apply UDEMD to medical knowledge graphs using Systemized Nomenclature of Medicine - Clinical Terms (SNOMED-CT)~\cite{snomed}. We show that unbalanced diffusion EMD can be used to find meaningful distances between patients which successfully clusters patients into different diagnosis categories, and allows us to find relationships between patient features. We also apply UDEMD to single cell RNA-sequencing data where we can model both cells as signals on gene interaction graphs or genes as signals on cell-similarity graphs. In cases where the gene regulatory network is well known, researchers have shown that affinity between cells can be computed as an earth mover's distance~\cite{huizing_optimal_2021, bellazzi_gene_2021}. We show that UDEMD runs orders of magnitude faster than the Sinkhorn and network simplex methods used in those works, while maintaining accuracy. In cases where the gene regulatory network is not well known, we model the transposed problem, deriving groupings of genes that function similarly by modeling genes as expression values over single cells. Here, we show that the UDEMD provides robust distances that recapitulate ground truth gene groupings in single cell data from peripheral blood mononuclear cells (PBMCs).


\section{Preliminaries}\label{sec:preliminaries}
In this section we review the Wasserstein metric, embedding based methods for approximating it, and unbalanced optimal transport.

The Wasserstein metric is a notion of distance between two measures $\mu, \nu$ on a measurable space $\Omega$ endowed with a metric $d(\cdot,\cdot)$ known as the ground distance. The primal formulation of the Wasserstein distance $W_d$, also known as the earth mover's distance, is defined as:
\begin{equation}\label{eq:primal}
    W_d(\mu, \nu) := \inf_{\pi \in \Pi(\mu, \nu)} \int_{\Omega \times \Omega} d(x, y) \pi(dx, dy),
\end{equation}
where $\Pi(\mu, \nu)$ is the set of joint probability distributions $\pi$ on the space $\Omega \times \Omega$, such that for any subset $\omega \subset \Omega$, $\pi(\omega \times \Omega) = \mu(\omega)$ and $\pi(\Omega \times \omega) = \nu(\omega)$. Also of interest is the entropy regularized Wasserstein distance~\cite{cuturi_sinkhorn_2013}, which reduces the computation to $O(n^2)$. This algorithm is extremely parallelizable, and works quite well even for a small number of iterations~\cite{backurs_scalable_2020}, and there are many works investigating how to scale this to larger problems.

However, when comparing a large number of signals (say $m$), we must solve the optimization for each pair of signals, i.e. $O(m^2)$ optimizations. For this reason, we turn to methods that approximate the dual of the Wasserstein metric, also known as the Kantorovich-Rubenstein dual formulation, which relies on witness functions. Many works optimize the cost over a modified family of witness functions such as functions parameterized by neural networks~\cite{arjovsky_wasserstein_2017, gulrajani_improved_2017, tong_fixing_2020}, functions defined over trees~\cite{indyk_fast_2003, le_tree-sliced_2019}, and wavelet bases~\cite{shirdhonkar_approximate_2008, gavish_multiscale_2010}. An efficient algorithm recently proposed is Diffusion EMD~\cite{tong_diffusion_2021}, it is based on a multi-scale representation of the signals. Indeed, it can be seen as a weighted average of the $L^1$ distances between two signals at different scales.

There are numerous formulations of \textit{unbalanced} optimal transport both to accommodate problems with unequal masses and to provide robustness to outlier points~\cite{peyre_computational_2019, balaji_robust_2020}. In general these can be formulated as a mixture between a pure optimal transport problem and a $\phi$-divergence. We focus on the formulation using the total variation, referred to as the TV-unbalanced problem:
\begin{equation}\label{eq:tvuw}
    \textup{TV-UW}_d(\mu, \nu) = \inf_s \left \{ \textup{W}_d(\mu + s, \nu) + \lambda \textup{TV}(\mu + s, \mu) \right \},
\end{equation}
where $\lambda = \min \{ \lambda_\mu,\lambda_\nu \}$ and $\lambda_\mu, \lambda_\nu$ control the relative cost of mass creation / destruction compared to transportation. Intuitively, we can think of \eqref{eq:tvuw} as minimizing over the ``teleporting'' mass $s$, that is too costly to transport. 

In the unbalanced optimal transport literature, most often considered is the KL-divergence formulation which can be solved efficiently in the case of entropic regularized problem \cite{chizat_unbalanced_2018, liero_optimal_2018, fatras_unbalanced_2021}, but is difficult to optimize stochastically as is possible in the balanced case, limiting scalability~\cite{genevay_learning_2018, fatras_learning_2020}. The TV-unbalanced problem (\eqref{eq:tvuw}) can be solved by adding a ``dummy point'' that is connected to every point with equal cost~\cite{caffarelli_free_2010, pele_fast_2009}. However, adding a dummy point removes the metric structure necessary for dual-based Wasserstein distances. It is not immediately obvious that \eqref{eq:tvuw} is efficiently computable while maintaining this structure. To address this issue, \cite{mukherjee_outlier-robust_2020} showed that the TV-unbalanced problem can be solved through cost truncation. Following their work, we will show that there is an embedding of distributions to vectors where the $L^1$ distance between vectors is equivalent to the $\textup{TV-UW}$ between the distributions.

\section{Unbalanced Diffusion Earth Mover's Distance}\label{sec:udemd}

While Diffusion EMD presented in~\cite{tong_diffusion_2021} can provide an earth mover's distance between graph signals, its formulation is not motivated by considering noisy signals on graphs or outliers, but rather geared to avoid high dimensional density estimation. Here, we focus on utilizing EMD to organize graph signals. Therefore, we are interested in distances that are immune to outlier spikes in the signals. While the multiscale smoothing proposed in~\cite{tong_diffusion_2021} is effective in handling noisy perturbation of the signals, it is less effective at dealing with outlier vertex components of the signal. However, as we show here, the construction can be adapted to consider unbalanced transport, which is essentially based on the idea that a more faithful earth mover's distance is given by a transport in which we ignore some of the mass -- particularly, mass that requires large transport costs. To incorporate this idea, we modify the formulation of \cite{tong_diffusion_2021} by only considering certain scales. This yields the Unbalanced Diffusion EMD (UDEMD), which is topologically equivalent to the total variation unbalanced Wasserstein distance. 

\begin{definition}\label{def:udemd}
The Unbalanced Diffusion Earth Mover's Distance (UDEMD) between two signals $\mu, \nu$ is
\begin{align}\label{eq:unbalanced_diffusion_emd_sum}
    \textup{UDEMD}_{\alpha, K}(\mu, \nu) &:= \sum_{v \in V} \sum_{k=0}^{K} \| g_{\alpha, k}(\mu(v)) - g_{\alpha, k}(\nu(v)) \|_1
\end{align}
where $0 < \alpha < 1/2$ is a meta-parameter used to balance long- and short-range distances, and
\begin{align}\label{eq:diffusion_emd}
    g_{\alpha, k}(\mu(v)) &:= 2^{-(K-k-1) \alpha} \big(\vmu^{({k+1})} - \vmu^{({k})}\big)
\end{align}
where $\vmu^{(t)}$ is short for $\vmu^{(t)}(v) = (\mP^{2^t} \mu)(v)$ and $K$ is the maximum scale considered.
\end{definition}
The scale $K$ relates to the unbalancing threshold (see Fig.~\ref{fig:ring_ground_dist} and discussion in Sec.~\ref{sec:unbalanced-theory}).  In practice, $\alpha$ is set close to 1/2, hence we drop the subscript and use the notation $\textup{UDEMD}_{K}$.



\subsection{Equivalence to (unbalanced) Wasserstein distance}
\label{sec:unbalanced-theory}

In \cite{mukherjee_outlier-robust_2020}, it was shown that truncated-cost optimal transport distances were equivalent to unbalanced Wasserstein distances, and that they are useful in outlier detection. However, there the proposed implementation used a truncated matrix with the standard Sinkhorn algorithm~\cite{cuturi_sinkhorn_2013}. Here we show a similar result for the Unbalanced Diffusion EMD from Def.~\ref{def:udemd}, i.e., showing that with scale truncation it is equivalent to an unbalanced Wasserstein distance. We first adapt Theorem 3.1 from \cite{mukherjee_outlier-robust_2020} in the following Lemma~\ref{lem: Lemma_UOT}, which will in turn be combined with Lemma~\ref{lem:udemd} to yield this result.

\begin{lemma}
\label{lem: Lemma_UOT}
The Wasserstein distance with a truncated ground distance $d_\lambda(x,y) = \min \{\lambda, d(x,y)\}$ for some constant $\lambda$ and distance $d$ is equivalent to a total variation unbalanced Wasserstein distance for some constant $\lambda$, i.e., $W_{d_\lambda}(\mu, \nu) =  \textup{TV-UW}_d(\mu, \nu)$.
\end{lemma}


The theory developed in \cite{tong_diffusion_2021} assumed that the support of the considered distributions was a closed Riemannian manifold. In such a case, Diffusion EMD will converge to a distance that is equivalent to the Wasserstein distance defined with the geodesic on the manifold. The following Lemma extends this theory to show that UDEMD (Def.~\ref{def:udemd}) will converge to a Wasserstein distance where the ground distance is a thresholded geodesic. 

\begin{lemma}\label{lem:udemd}
$\textup{UDEMD}_{K}(\mu, \nu)$ approximates a metric equivalent to the Wasserstein distance $W_{d_\lambda}(\mu, \nu)$, defined as in Lemma~\ref{lem: Lemma_UOT}, with the ground distance being a truncated geodesic distance on the manifold, i.e., $d_{\lambda}(x,y) = \min \{\lambda, \rho(x,y) \}$ for $ \lambda > 0$.
\end{lemma}
\begin{proof}

We present a proof sketch here; the main part of the proof follows the same lines as in Corollary 3.1 of \cite{tong_diffusion_2021}. In Def.~\ref{def:udemd}, an anisotropic kernel $\mP$ is used, which can be shown to converge to the Heat kernel on a Riemannian manifold (see \cite{coifman_diffusion_2006-1}, Prop.~3). In \cite{leeb_holderlipschitz_2016}, it is shown that the construction of Def.~\ref{def:udemd} using the Heat Kernel will converge to a metric that is equivalent to the Wasserstein with ground distance $\min\{1,\rho(x,y)^{2\alpha}\}$, where $\rho$ is the geodesic on the manifold. Because the metrics $\min\{1,\rho(x,y)^{2\alpha}\}$ and $\min\{\lambda,\rho(x,y)^{2\alpha}\}$ are equivalent for $\lambda >0$, the Wasserstein distances induced by these metrics are also equivalent.
\end{proof}

By combining Lemmas~\ref{lem: Lemma_UOT} and \ref{lem:udemd}, we have that the UDEMD from Def.~\ref{def:udemd} approximates a metric equivalent to an unbalanced optimal transport metric. Formally, using the equivalence notation from~\cite{tong_diffusion_2021}, we have $\textup{UDEMD}_{K}(\mu, \nu) \simeq  \textup{TV-UW}_d(\mu, \nu)$.
We note that while our result here establishes a relation between these two metrics, it does not directly quantify the relation between the $\lambda$ and $K$. We leave careful theoretical and rigorous study of this relation to future work\footnote{We deem the required nuanced treatment of scaling constants in equivalence bounds out of scope in this work.}, but mention here that we observe empirically, as shown in Fig.~\ref{fig:ring_ground_dist}, the choice of $K$ indeed acts in a similar way to the threshold $\lambda$ on the ground distance.


\input{figures/ring}

\begin{algorithm}
    \caption{$\textup{UDEMD}(\mA, \vmu, K, \alpha) \rightarrow \vb$}
    \label{alg:cheb_embedding}
\begin{algorithmic}
    \STATE {\bfseries Input:} $n \times n$ graph adjacency $\mA$, $n \times m$ distributions $\vmu$, maximum scale $K$, and snowflake constant $\alpha$.
    \STATE {\bfseries Output:} $m \times (K+1) n$ distribution embeddings $\vb$
    \STATE $\mP = \mD^{-1} \mA$
    \STATE $\vmu^{(2^0)} \leftarrow \vmu$
    \FOR{$k=1$ {\bfseries to} $K$}
        \STATE $\vmu^{(2^k)} \leftarrow \mP^{2^k} \vmu^{(2^{k-1})}$
        \STATE $\vb_{k-1} \leftarrow 2^{(K-k-1)\alpha}(\vmu^{(2^k)} - \vmu^{(2^{k-1})})$
    \ENDFOR
    \STATE $\vb_{K} \leftarrow \vmu^{(2^K)}$
    \STATE $\vb \leftarrow [\vb_0, \vb_1, \ldots, \vb_K]$
    \STATE $\tilde{\vb} \leftarrow \textup{Subsample}(\vb)$
\end{algorithmic}
\end{algorithm}

To compute the UDEMD defined in Def.~\ref{def:udemd}, we present Alg.~\ref{alg:cheb_embedding} with time complexity $O(2^K m |E|)$, which is similar to algorithms used in graph neural networks. Our algorithm scales well with the size of the graph, the number of distributions $m$ and number of points $n$, but poorly with the maximum scale $K$. We note that the maximum scale considered for Diffusion EMD was of order $O(\log |V|)$, derived from the convergence of the heat kernel to its steady state. Here, on the other hand, we decouple the tuning of $K$ and find that a much smaller maximum scale suffices, and in fact (as discussed in Sec.~\ref{sec:unbalanced-theory}) corresponds to a well characterized unbalanced earth mover's distance on the underlying geometry of the graph. This leads to Alg.~\ref{alg:cheb_embedding} emphasizing preferable scaling properties for small $K$, and easily accelerated by computation on GPUs.


\section{Results}\label{sec:results}

In this section, we show that UDEMD is an efficient and robust method for measuring distances between graph signals and then using the distances to find embeddings and organization of the signals (often entities such as patients). We compare UDEMD to a GPU implementation of numerically stabilized Sinkhorn optimization that includes minibatching of sets of distributions. However, despite this, this method runs out of memory when there are beyond ~10,000 nodes in the graph. We note that all methods of this type require solving $m^2$ optimizations, even when looking for nearest neighbors. Unless otherwise noted, we set $K=4$ and $\alpha=1/2$.

\textbf{Spherical data test case} To test the speed and robustness of UDEMD we begin with a dataset where we have knowledge of the intrinsic ground distances and can vary the number of points and distributions. For this dataset we sample $m$ Gaussian distributions with means distributed uniformly on the unit sphere with 10 points each for a total of $n=10 m$ points. We add a random noise spike at a uniformly random location on the sphere to check robustness to this type of noise. The goal is to predict the neighboring distributions on the sphere. We find that UDEMD is significantly more scalable and find that there is a sweet spot in terms of $K$ at $K=4$ for this dataset. The UDEMD with $K=4$ performs significantly better than the balanced Diffusion EMD case with this type of noise. This supports the claim that setting $K \ll O(\log n)$ is beneficial in real world datasets. UDEMD also outperforms the graph-TV distance as it is both faster and more accurate at $K=1$, and more accurate overall. 

\input{figures/sphere}

\textbf{Single-cell data with cells as signals over gene graphs} We consider 206 cells from the K562 human lymphoblast cell line as signals over a known 10000-node gene graph in single-cell RNA seq data~\cite{liu_deconvolution_2019}.  We measure the distance between cells based on their transport on this gene graph. This was recently independently proposed by \cite{huizing_optimal_2021} and \cite{bellazzi_gene_2021}, who showed that OT over the gene graph can provide better distances between cells than Euclidean measures.  We measure the performance of these methods based on how well the resulting distance matrix between cells matches the clusters according to four scores: Silhouette score, the adjusted rand index (ARI), the normalized mutual information (NMI), and the adjusted mutual information (AMI). In Fig.~\ref{fig:gene_graph}, we see that UDEMD performs almost as well as Sinkhorn-OT, and much better than the Euclidean and total variation distances that do not take into account the gene graph as well as much faster than Sinkhorn-OT, scaling almost as well as Euclidean distances due to the embedding. Note however, that using balanced transport (see Fig.~\ref{fig:sphere} right) degrades the accuracy. The balanced transport compared here is the original Diffusion EMD from \cite{tong_diffusion_2021} which is not a thresholded distance, and thus noise in the data are able to perturb the accuracy of the distances.

\input{figures/gene_computation}

\textbf{Single-cell data with genes as signals over cell graphs} Next we applied our approach to 4,360 peripheral blood mononuclear cells measured via single cell RNAseq publicly available on the 10X platform. We consider three curated gene sets that are explanatory for this dataset. We compare the distances between genes using UDEMD, Euclidean, total variation and Sinkhorn-OT distances. We can see that the genes canonical for monocytes (orange), T cells (green) and B cells (blue) all appear to be closely positioned to one another and separate between the groups in our embedding in contrast to a Euclidean distance embedding of the genes where the clusters are less clear (Fig.~\ref{fig:pbmc}a). Visualizing the UDEMD distance between our 46 genes in a heatmap, we can identify the three clusters as dark blocks of low distance (Fig.~\ref{fig:pbmc}b). This result is quantified in (Fig.~\ref{fig:pbmc}c), where UDEMD performs the best on 3/4 metrics. Last, we tried to see how diffusion time scale impacted silhouette score, identifying that score maximized at a timescale of $K=10$ and did not improve with higher scales.


\input{figures/PBMC_Single_Cell}

\textbf{A patient concept knowledge graph}
We consider a knowledge graph constructed from medical concepts captured in clinical documentation and reporting. SNOMED-CT is a widely used collection of terms and concepts with defined relationships considered to be an international standard for medical concepts captured from the electronic health record. SNOMED-CT has a pre-defined knowledge graph with concept-relation-concept triplets, which we subset to the Clinical Findings concept model (version 3/2021). We used 52,150 discharge summaries from MIMIC-III, which contain all information about a patient's hospital course and extracted concepts using MetaMap (version 2018)~\cite{metamap}. These medical concepts were then used as signals on the SNOMED-CT knowledge graph, which link all relevant concepts together. The metadata used to label patients included primary diagnosis, a physician-designated diagnosis which was stored separately in the MIMIC-III database.

One of the advantages of the UDEMD-based embedding is the identification of clinically meaningful overlaps that may not be apparent from the single primary diagnosis recorded in the database. Patients with a primary diagnosis of intracranial bleeding (bleeding in the brain) can also have primary brain masses and tumors. Compared to the spurious fragmentation of patients with the same diagnosis of intracranial bleeding into several clusters in the TV embedding, UDEMD consolidated patients with the same diagnosis of intracranial bleeding and specifically grouped those that may have had bleeding due to a primary brain mass or tumor (See Fig.~\ref{fig:snomed}b). Interestingly, UDEMD also identified patients who were predicted to have intracranial bleeding to have the diagnosis of stroke with higher accuracy, reflecting consistency with the fact that a subtype of stroke (hemorrhagic) is due to intracranial bleeding (Fig.~\ref{fig:snomed}d).


\input{figures/snomed}

\section{Conclusion}

In this work we explored the use of earth mover's distance to organize signals on large graphs. We presented an unbalanced extension of Diffusion EMD, which we showed approximates a distance equivalent to the total variation unbalanced Wasserstein distance between signals on a graph. We showed how to compute nearest neighbors in this space in time log-linear in the number of nodes in the graph and the number of signals. Finally, we demonstrated how this can be applied to entities which can be modeled as signals on graphs between genes, cells, and biomedical concepts.


\clearpage

\balance

\bibliographystyle{IEEE}
\bibliography{main}


\end{document}

%% file: math_commands.tex
\usepackage{amsmath,amsfonts,bm,mathtools}

\def\eqref#1{Eq.~\ref{#1}}

\def\1{\bm{1}}

\def\vmu{{\bm{\mu}}}

\def\vb{{\bm{b}}}

\def\mA{{\bm{A}}}

\def\mD{{\bm{D}}}

\def\mP{{\bm{P}}}

\DeclareMathAlphabet{\mathsfit}{\encodingdefault}{\sfdefault}{m}{sl}
\SetMathAlphabet{\mathsfit}{bold}{\encodingdefault}{\sfdefault}{bx}{n}

\usepackage{amsthm}

\DeclarePairedDelimiter\abs{\lvert}{\rvert}
\DeclarePairedDelimiter{\norm}{\lVert}{\rVert} 
\makeatletter
\let\oldabs\abs
\def\abs{\@ifstar{\oldabs}{\oldabs*}}

\let\oldnorm\norm
\def\norm{\@ifstar{\oldnorm}{\oldnorm*}}
\makeatother

\newtheorem{definition}{Definition}
\newtheorem{lemma}{Lemma}


%% file: figures/ring.tex
\begin{figure}
\centering
\begin{subfigure}{.45\linewidth}
  \centering
  \includegraphics[width=.99\linewidth]{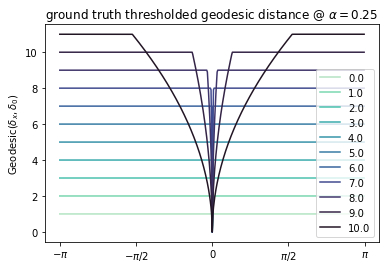}
  \vspace{-15pt}
  \caption{Thresholded geodesic}
  \label{fig:ring:sfig1}
\end{subfigure}%
\begin{subfigure}{.45\linewidth}
  \centering
  \includegraphics[width=.99\linewidth]{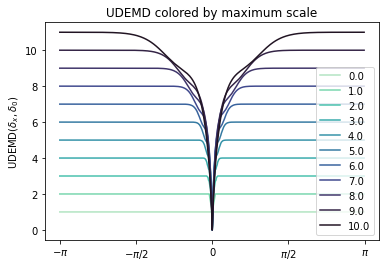}
  \vspace{-15pt}
  \caption{UDEMD between diracs}
  \label{fig:ring:sfig2}
\end{subfigure}

\vspace{-5pt}

\caption{On a ring graph $n=500$ compares the UDEMD to the thresholded ground distance, this suggests that UDEMD closely approximates the thresholded ground distance with $\lambda \approx 2^K$.}

\vspace{-15pt}

\label{fig:ring_ground_dist}
\end{figure}

%% file: figures/sphere.tex
\begin{figure}[ht]
\centering
\includegraphics[width=.95\linewidth]{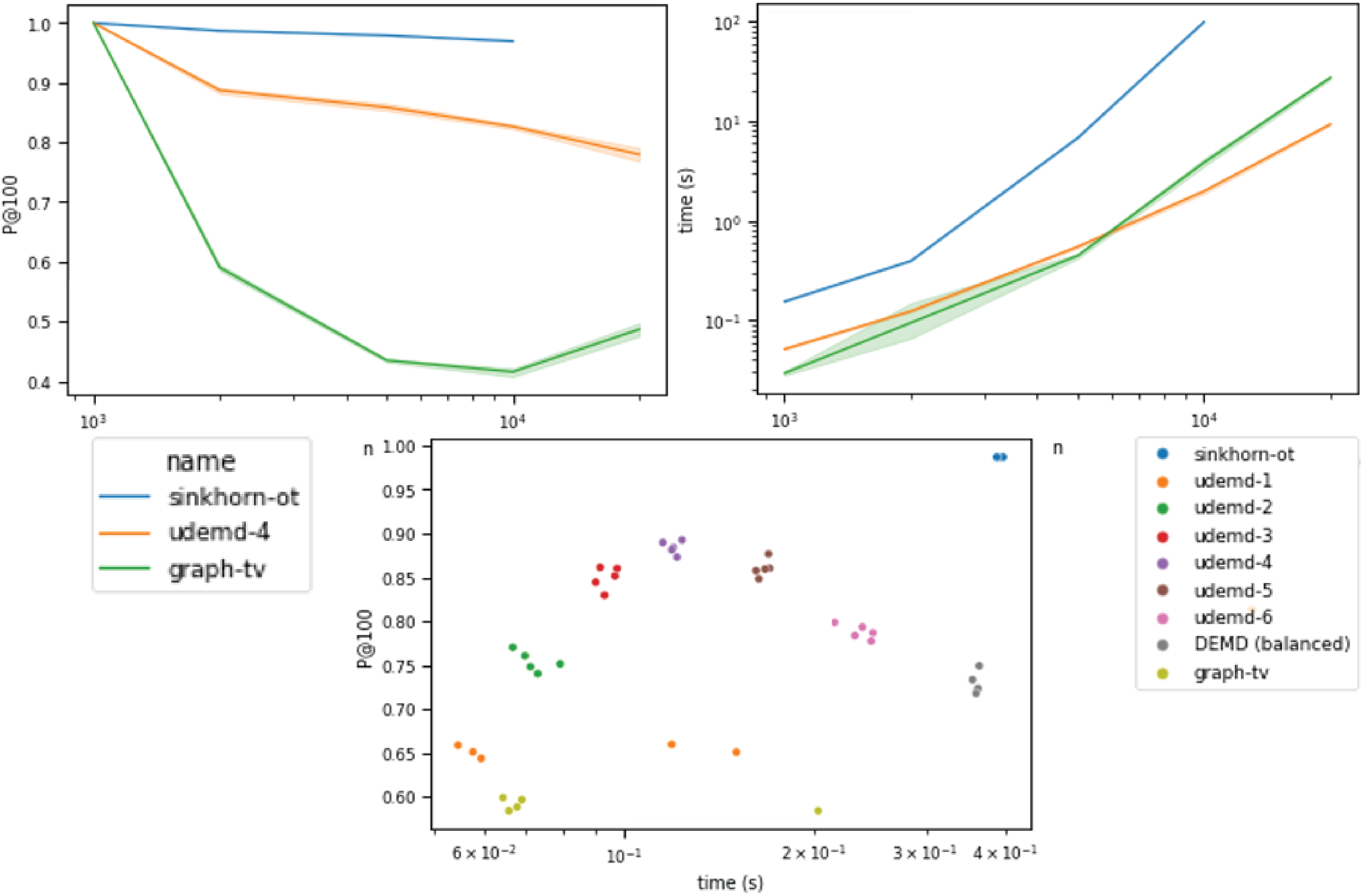}
\vspace{-2mm}
\caption{UDEMD is more scalable than Sinkhorn-OT and performs better than graph total variation. (left) Shows performance as measured by P@100, the fraction of the 100 nearest neighbors predicted correctly, against problem size. (middle) Shows time against problem size, and (right) shows performance vs. time on a problem size of $n=2000$ for different choices of $K$.}
\vspace{-2mm}
\label{fig:sphere}
\end{figure}

%% file: figures/gene_computation.tex
\begin{figure}[ht]
\centering
\begin{subfigure}{.47\linewidth}
  \centering
  \includegraphics[width=.99\linewidth]{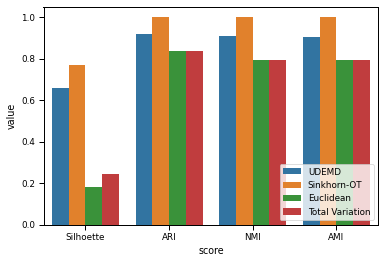}
  \vspace{-15pt}
  \caption{Performance metrics}
  \label{fig:gene:sfig1}
\end{subfigure}%
\begin{subfigure}{.47\linewidth}
  \centering
  \includegraphics[width=.99\linewidth]{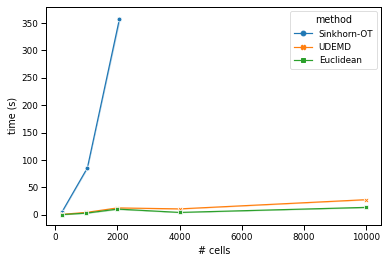}
  \vspace{-15pt}
  \caption{Computation comparison}
  \label{fig:gene:sfig2}
\end{subfigure}
\vspace{-5pt}
\caption{UDEMD achieves better clustering than Euclidean and total variation (TV) distances, and performs similarly well to Sinkhorn-OT but is much more scalable with similar scalability to Euclidean and TV distances. (a) performance in terms of Silhouette score, ARI, NMI, and AMI (b) computation time vs. problem size.}
\label{fig:gene_graph}
\end{figure}


%% file: figures/PBMC_Single_Cell.tex
\begin{figure}[ht]
\begin{center}
\centerline{\includegraphics[width=\columnwidth]{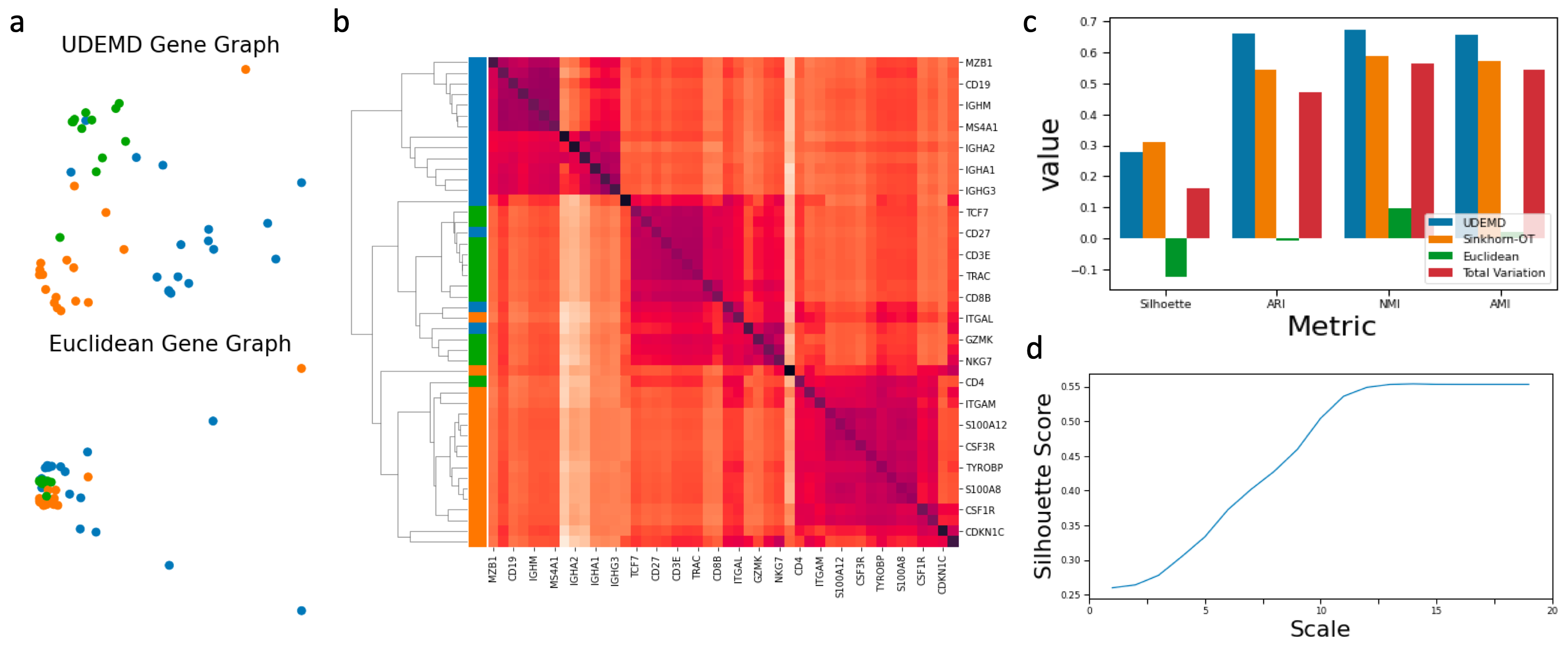}}
\vspace{-2mm}
\caption{(a) Visualization of gene graphs of 46 genes canonical for different cell types using UDEMD and Euclidean ground distances (blue for B cells, orange for monocytes and green for T cells), (b) heat map of gene distances (c) clustering performance (d) silhouette score vs. maximum diffusion scale $K$.}
\vspace{-9mm}
\label{fig:pbmc}
\end{center}
\end{figure}


%% file: figures/snomed.tex
\begin{figure}[ht]

\begin{center}
\centerline{\includegraphics[width=\columnwidth]{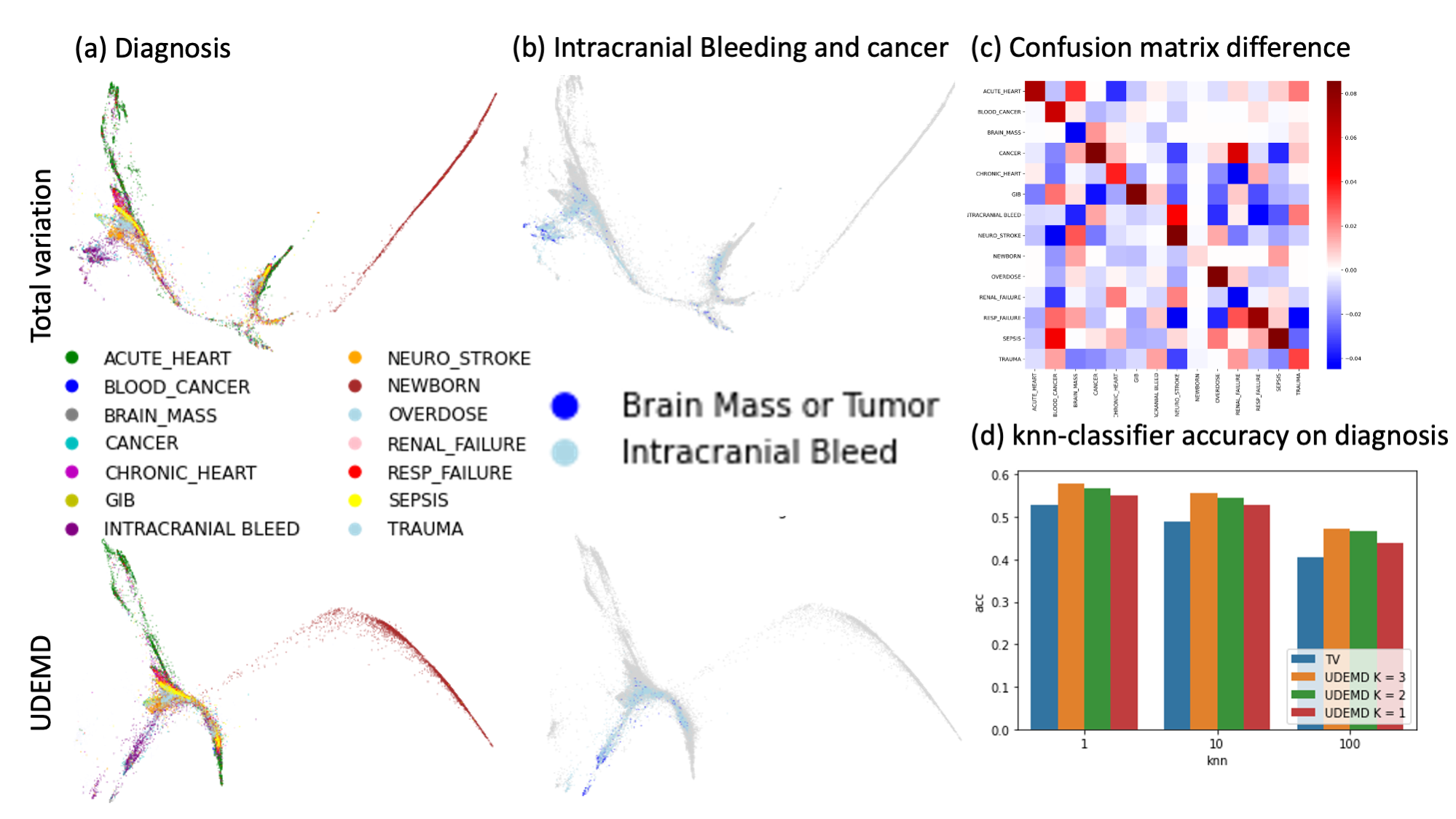}}
\vskip -0.1in
\caption{Embeddings of patients modeled as signals over the SNOMED-CT graph using TV distance (a top) and using UDEMD distance (a bottom), colored by patient diagnosis. UDEMD better organizes the space as noted by selected terms in (b), difference of confusion matrices in (c) and k-nearest neighbors classification accuracy on the diagnosis in (d). In (b) note that the TV embedding (top) creates a spurious separation (due to noise in the signal) between subsets of patients who display intracranial bleeding that is not distinguished by diagnosis. On the other hand the UDEMD embedding (bottom) shows a continuum of patients with this diagnosis. The same holds for patients with brain mass or tumor shown in green. } 
\label{fig:snomed}
\end{center}
\vskip -0.4in
\end{figure}